\documentclass[journal,twoside,web]{ieeecolor}
\usepackage{lcsys}
\usepackage{cite}
\usepackage{amsmath,amssymb,amsfonts,amstext}
\usepackage{graphicx}
\usepackage{subcaption}
\usepackage{textcomp}
  \def\BibTeX{{\rm B\kern-.05em{\sc i\kern-.025em b}\kern-.08em
     T\kern-.1667em\lower.7ex\hbox{E}\kern-.125emX}}
\usepackage[utf8]{inputenc}
\usepackage{babel}
\usepackage{balance}
\usepackage{comment}
\usepackage{bbold}
\usepackage{mathtools}
\usepackage{float}

\let\theoremstyle\relax
\usepackage{amsthm}

\usepackage{algorithm}
\usepackage{algpseudocode}

\theoremstyle{definition}
\newtheorem{assumption}{Assumption}
\newtheorem{definition}{Definition}
\newtheorem{remark}{Remark}
\newtheorem{problem}{Problem}

\theoremstyle{plain}

\newtheorem{theorem}{Theorem}
\newtheorem{lemma}{Lemma}

\usepackage{algorithm}
\usepackage{algorithmicx} 
\usepackage{algpseudocode} %

\makeatletter
\let\NAT@parse\undefined
\makeatother
\usepackage{hyperref}
\hypersetup{
  colorlinks=true,
    linkcolor= blue,
    allcolors=blue,
    citecolor = blue,
    filecolor=black,      
    urlcolor=blue,
    }

\title{A Q-learning Approach for Adherence-Aware Recommendations}

\author{Ioannis Faros$^{1}$, \textit{Student Member, IEEE}, Aditya Dave$^{2}$, \textit{Member, IEEE},\\Andreas A. Malikopoulos$^{1,2}$, \textit{Senior Member, IEEE}
    \thanks{This research was supported by NSF under Grants CNS-2149520 and CMMI-2219761.}
    \thanks{$^{1}$Systems Engineering, Cornell University, Ithaca, NY 14850, USA.} 
    \thanks{$^{2}$School of Civil and Environmental Engineering, Cornell University, Ithaca, NY 14850, USA. {\tt\small email: \{if74,adidave,amaliko\}@cornell.edu}}
}

\date{February 2022}
\begin{document}

\maketitle
\thispagestyle{empty}
\begin{abstract}
In many real-world scenarios involving high-stakes and safety implications, a human decision-maker (HDM) may receive recommendations from an artificial intelligence while holding the ultimate responsibility of making decisions. 
In this letter, we develop an ``adherence-aware Q-learning" algorithm to address this problem. The algorithm learns the ``adherence level" that captures the frequency with which an HDM follows the recommended actions and derives the best recommendation policy in real time. We prove the convergence of the proposed Q-learning algorithm to the optimal value and evaluate its performance across various scenarios.
\end{abstract}


\section{Introduction}
Decisions driven by artificial intelligence have recently found applications in complex cyber-physical systems \cite{Malikopoulos2022a} such as transportation \cite{Nishanth2023AISmerging,Malikopoulos2020}, finance \cite{d2019promises}, and healthcare \cite{lin2021does}.
However, decisions involving high-stakes or safety-critical applications \cite{dave2023worstcase} are often ultimately taken by human decision-makers (HDMs) under advice from an artificial intelligence algorithm. 
Since this is at the discretion of the HDM, the algorithm's recommendations may not be followed at every instance of time \cite{balakrishnan2022improving}.
The phenomenon of unexpected decisions influencing the performance of such ``expert-in-the-loop" systems has garnered increasing interest in recent years, 
primarily in the fields of operations \cite{kesavan2020field}, human trust on machines \cite{logg2019algorithm}, finance \cite{d2019promises}, and healthcare services \cite{lin2021does}.

Many research efforts have focused on the factors influencing the adherence of HDMs to recommendations. In \cite{dietvorst2018overcoming}, it was established that HDMs usually prefer following recommendations that match their comfort and expertise and may ignore recommendations that contradict their opinions. The authors in \cite{logg2019algorithm} provide evidence for the hypothesis that HDMs trust their knowledge rather than an algorithm.
A similar phenomenon for partial adherence is also observed when humans recommend actions to other humans, e.g., medical advice \cite{sabate2003adherence}.
Conversely, an algorithm may not be able to account for real-life limitations faced by humans when implementing an action \cite{van2010ordering}.
Thus, an HDM may trust their judgment to navigate such situations and disregard algorithmic recommendations 
\cite{sun2022predicting}.
As a consequence of these factors, algorithmic recommendations to HDMs can perform significantly worse than anticipated \cite{sabate2003adherence}.

Two main approaches have been proposed in the literature to address performance degradation: \textit{(1) Increasing adoption:} Many research efforts have focused on increasing the adoption of recommendations among HDMs \cite{dietvorst2015algorithm}.
In \cite{bastani2021improving}, reinforcement learning techniques were utilized to increase HDM adherence, assuming they are likelier to follow recommendations close to their baseline strategy. In \cite{jacq2022lazy}, a Lazy
Markov decision processes (MDPs) formulation was proposed to improve the alignment of recommendations with the baseline strategy. 
\textit{(2) Incorporating adherence awareness:} A more recent approach is to improve recommendations by considering partial adherence within the algorithm's problem. In \cite{grand2022best}, an adherence-aware MDP was analyzed, and a value iteration algorithm was proposed to improve the performance of recommendations. 
However, this result relies upon prior knowledge of the HDM's adherence level and system dynamics. 
To bring this framework closer to real-world implementation, we require further insights into adherence-aware recommendations that can be obtained with incomplete prior knowledge of dynamics and adherence levels.

In this letter, we address this challenge by incorporating partial adherence into reinforcement learning \cite{hung2016q}, \cite{brunke2022safe} by proposing an ``adherence-aware Q-learning algorithm."
Specifically, we consider an MDP comprising an HDM that implements actions to influence the evolution of an unknown environment. 
The HDM's strategy for generating actions combines an algorithmic recommendation and a baseline strategy.
The adherence level of the HDM is not known to the algorithm a priori. 
The algorithm and the HDM share the same objective, and thus, our formulation aims to compute in real time the best-recommended actions to maximize an expected total discounted reward. 

Our main contributions in this letter are the (1) introduction of an adherence-aware Q-learning algorithm (Algorithm \eqref{eq_adh_q}) to compute an optimal control law, (2) convergence for the algorithm to the optimal Q-function (Theorem \ref{theorem1}), and (3) establishment of the advantages of our approach against baseline policies and classical Q-learning using numerical examples (Section \ref{sec:example}).

The remainder of the letter proceeds as follows. In Section \ref{sec:pf}, we present our formulation, the definitions of the adherence-aware Q-learning function, and the updating rule of the adherence level. In Section \ref{ts}, we prove the convergence of the proposed algorithm. In Section \ref{sec:example}, we demonstrate our result in an inventory control problem, and in Section \ref{sec:conclusion}, we draw concluding remarks.
\section{Modeling Framework} \label{sec:pf}

\subsection{Problem Formulation}

We consider a system comprising an artificial intelligence or algorithm recommending actions to an HDM. In turn, the HDM implements actions to influence the evolution of a dynamic environment. At each instance of time, the HDM can select to either follow the recommendation provided by the algorithm or select an action using a baseline law. 
The evolution of the state of the environment is modeled as an infinite horizon discounted MDP $\mathcal{S} = (\mathcal{X},\mathcal{U},P,R,\lambda) $, where $\mathcal{X}$ is a finite set of states and $\mathcal{U}$ is a finite set of actions. 
At any time $t\in \mathbb{N} = \{0,1,\dots\}$, the state of the system is denoted by the random variable $X_t \in \mathcal{X}$ and the action input to the MDP is the random variable $U_t \in \mathcal{U}$.
The function $P: \mathcal{X} \times \mathcal{X} \times \mathcal{U} \to [0,1]$ yields the transition probability for all $t$ as $P(x_{t+1}~|~x_t,u_t) = p(X_{t+1} = x_{t+1}~|~X_t = x_t, U_t = u_t)$ from any realized state $x_t \in \mathcal{X}$ and realized action $u_t \in \mathcal{U}$ to the next realized state $x_{t+1} \in \mathcal{X}$.
The function $R: \mathcal{X}\times \mathcal{U} \to \mathbb{R}$ yields the reward $R(X_t,U_t)$ for all $t \in \mathbb{N}$, and $\lambda \in (0,1]$ is a discount factor applied to future rewards when measuring performance.

In the MDP $\mathcal{S}$, the actions are implemented by an HDM with access to a \emph{baseline} law $g^{\text{b}}: \mathcal{X} \to \mathcal{U}$ that belongs to a set $G$ of stationary Markovian laws.
The HDM also receives an action recommendation at each instance of time from an algorithm using a \textit{recommendation} law $g^{\text{r}}: \mathcal{X} \to \mathcal{U}$ from the set of feasible laws $G$. 
Subsequently, we consider that the HDM implements the recommended action with a probability $\theta \in [0,1]$ and the baseline action with a probability $1-\theta$.
Thus, the HDM follows a mixed \textit{actual} control law given by a convex combination of the recommended and baseline laws as follows:
\begin{equation}\label{eq_actual}
    g^{\text{a}}
    = \theta {\cdot} g^{\text{r}} + (1 - \theta){\cdot}g^{\text{b}}.
\end{equation}
\begin{figure}
    \centering
    \includegraphics[scale = 0.85, keepaspectratio]{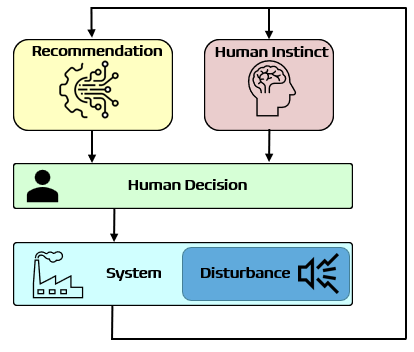}
    \caption{Process of the actual law implementation.}
    \label{fig:HDM}
\end{figure}
The process of an HDM taking an action using the baseline and the recommendation law is illustrated in Fig. \ref{fig:HDM}.
The probability $\theta$ is a representation of the \emph{adherence} level of the HDM to a recommended action, i.e., it captures the fact that an HDM does not systematically implement the latter.
The objective of an HDM is to maximize the expected total discounted reward
    \begin{equation} \label{problem}
       J(g^{\text{a}};x) = \lim_{T \to \infty} \mathbb{E}^{g^{\text{a}}}\Bigg[\sum_{t=1}^{T} \lambda^{t+1} {\cdot} R(X_t,U_t) \Big| X_1 = x \Bigg],
    \end{equation}  
where $\mathbb{E}^{g^{\text{a}}}[\cdot]$ is the expectation with respect to the distributions of all random variables generated using the actual control law.

\begin{problem} \label{problem_1}
    The objective is to derive the optimal control law $g^{r*}$ such that $g^{\text{a}*} = \theta g^{\text{r}*} + (1-\theta)g^{\text{b}}$ satisfies
    \begin{equation}
        J(g^{\text{a}*}) \geq J(g^{\text{a}}), \quad \forall g^{\text{a}} \in G^{\text{a}}(g^{\text{b}}),
    \end{equation}
where $G^{\text{a}}(g^{\text{b}})$ is the set of all possible control laws for a given baseline law $g^{\text{b}}$ and $\theta$.
\end{problem}

\begin{remark}
    The term $U_t$ refers to the action implemented by the HDM at any $t \in \mathbb{N}$. Additionally, we denote the recommended action as $U_t^{\text{r}} = g^{\text{r}}(X_t)$ and the baseline action as $U_t^{\text{b}} = g^{\text{b}}(X_t)$ for all $t$.
\end{remark}

In our modeling framework, we impose the following assumptions :

\begin{assumption}
    The stationary adherence level $\theta$ and transition matrix $P$ may be unknown. However, the baseline law $g^{\text{b}}$ is known a priori. 
\end{assumption}
The baseline law $g^{\text{b}}$ can be determined by observing the system over a finite number of instances since the sets of states and actions are finite. Thus, we consider the law to be known in our problem formulation.
\begin{assumption}\label{ass_2}
    The baseline law $g^{\text{b}}$ is deterministic.
\end{assumption}
The baseline law must be deterministic to recognize when the HDM implements the baseline action as opposed to when they implement the recommended action.
This assumption is satisfied by the optimal solutions to regular MDPs. An extension of our results for mixed baseline laws is a direction of future research.

\subsection{Preliminary Results}
    In Problem \ref{problem_1}, an optimal value function is given by the fixed point of the following Bellman-like equation for all possible states $x \in \mathcal{X}$ \cite{grand2022best}:
\begin{equation}\label{bellman}
    \begin{aligned}
         V(x) = \max_{u^{\text{r}}} \theta {\cdot} \sum_{x' \in \mathcal{X}}\Big(R(x,u^{\text{r}}) + P(x'\,|\,x, u^{\text{r}}){\cdot}V(x')\Big) + \\ (1-\theta) {\cdot} \sum_{x' \in \mathcal{X}}\Big(R(x,g^{\text{b}}(x)) + P(x'\,|\,x,g^{\text{b}}(x)){\cdot}V(x')\Big),
    \end{aligned}
\end{equation}
where $x'$ denotes a realization of the next state, $u^{\text{r}}$ the recommended action, $u^{\text{b}} =g^{\text{b}}(x) $ the base action for a given state $x \in \mathcal{X}$. 
Here, we use a time-invariant notation because \eqref{bellman} holds for all $t \in \mathbb{N}$ and all possible states in $\mathcal{X}$.
When the system dynamics and adherence level $\theta$ are known, we can use value iteration with \eqref{bellman} to derive an optimal recommendation law. 

In our formulation, we consider that these functions and quantities may be unknown. Next, we provide a Q-learning algorithm to compute the optimal recommendation law in Problem \ref{problem_1}. 

\section{Q-Learning Approach} \label{ts}
In this section, we propose an \textit{adherence-aware Q-learning} algorithm that can learn a Q-function for Problem \ref{problem_1} and subsequently use this to derive an optimal recommendation law. We also prove the convergence of the proposed algorithm to the optimal Q-function. We begin by constructing an unbiased estimate for the unknown $\theta$ using a point estimator of $\theta$ that considers the action implemented by the HDM. 
The tuple $(\theta, 1-\theta)$ can be interpreted as the probability distribution of a sequence of identical and independently distributed random variables $\{Y_t\}_{t \in \mathbb{N}}$, where each $Y_t \in \{0,1\}$ indicates whether the HDM implements the recommended action at time $t$, i.e., $Y_t = 1$ when $U_t = g^{\text{r}}(X_t)$ and $Y_t = 0$ when $U_t = g^{\text{b}}(X_t)$.
To construct this estimator for any sample size $n \in \mathbb{N}$, we first collect a sample $s^t$ comprising observations of the random variables $Y_0, \dots, Y_n$. Then, we utilize the point estimator to approximate $\theta$ as $\theta^t = s^t/n$. 
To carry out this procedure online during Q-learning, we can write the update rule for the $n+1$-th estimate $\theta_{t+1}$ of $\theta$ from the $n$-th estimate $\theta^t$ as  
    \begin{equation}
         \theta_{t+1} = \frac{\theta_t {\cdot} n + \mathbb{1}[Y_{t+1} = 1]}{n + 1},
    \end{equation}
    where $\mathbb{1}[\cdot]$ is the indicator function.
    Using standard arguments, we can show that this update rule will converge to the true $\theta$ value \cite{lehmann2006theory} in \eqref{eq_actual}, i.e.,
    \begin{equation}\label{expectation}
        \mathbb{E}[\theta_{t+1}] = \theta.
    \end{equation}
Using the update rule for $\theta$, we propose our adherence-aware Q-learning algorithm to learn the value function in \eqref{bellman} as
    \begin{equation} \label{eq_adh_q}
        \begin{aligned}
            & Q(x,u) \leftarrow  Q(x,u) + \alpha(x,u) {\cdot} \Big\{ \theta_{t+1} {\cdot}\Big[r(x,u) \\
            &+\lambda {\cdot} \max_{g^{\text{r}}}Q(x',g^{\text{r}}(x')) \Big] + (1-\theta_{t+1}) {\cdot} \Big[ r(x,u) \\
            &+ \lambda {\cdot} Q(x',g^{\text{b}}(x')) \Big] \
            - Q(x,u) \Big\},
        \end{aligned}
    \end{equation}
where $x' \in \mathcal{X}$ is the next state, $x \in \mathcal{X}$ is the current state, $u \in \mathcal{U}$ is the current action, and $\alpha(x,u)$ is the learning rate. We use time-invariant notation as this algorithm holds for all $t \in \mathbb{N}$ and all states in $\mathcal{X}$ and actions in $\mathcal{U}$. The term $Q(x,u)$ in the RHS is the current value for a given state $x$ and action $u$. The term $\max_{g^{\text{r}}}Q(x',g^{\text{r}}(x'))$ is the estimate of optimal future value, and $Q(x',g^{\text{b}}(x'))$ is the future value given the baseline law. Next, we show the adherence-aware Q-learning algorithm in procedural form.


\begin{algorithm} 
\caption{Adherence-aware Q-learning algorithm}\label{alg:cap}
\textbf{Algorithm parameters}: step size $\alpha \in (0,1]$, small $\epsilon > 0$, $\lambda \in [0,1]$.
Consider a baseline law $g^{\text{r}}(x) = u^{\text{r}}$.\\
\textbf{Initialize}: Q(x,u), for all $x \in \mathcal{X}, u \in \mathcal{U}$, arbitrarily.
\begin{algorithmic}
\For{each episode}
    \State \text{Initialize} $x$
    \For{each iteration}
        \State Choose $u$ using a law derived from $Q$ (e.g., $\epsilon$-greedy)
        \State Take action $u$, compute $r$, $x'$
        \State Update $Q$ using Equation \eqref{eq_adh_q}
        \State $x \leftarrow x'$
    \EndFor
\EndFor
\end{algorithmic}
\end{algorithm}
\vspace{-10pt}
\subsection{Proof of convergence} \label{res}
In this subsection, we prove the convergence of the adherence-aware Q-learning algorithm to the optimal value function. 
Before we prove convergence, we first define an adherence-aware operator for our Bellman-like recursion, and we prove that this is a contraction mapping. 
\begin{definition} 
    The adherence-aware operator $\mathcal{J}: [\mathcal{X}\times \mathcal{U} \to \mathbb{R}] \to  [\mathcal{X}\times \mathcal{U} \to \mathbb{R}]$ for any bounded $q: \mathcal{X}\times \mathcal{U} \to \mathbb{R}$ is
    \begin{multline}\label{oper}
            (\mathcal{J}q)(x_t,u_t) = \theta_{t+1} {\cdot} \hspace{-5pt} \max_{u_t^{\text{r}}}\sum_{X_{t+1} \in \mathcal{X}} \hspace{-5pt} P(x_{t+1} | x_t, u_t^{\text{r}})  {\cdot}\Big[ R(x_t, u_t^{\text{r}}) \\
           +\lambda {\cdot}q(x_{t+1},u_t^{\text{r}})  \Big] +(1-\theta_{t+1}) {\cdot} \hspace{-5pt} \sum_{x_{t+1} \in \mathcal{X}} P(x_{t+1} | x_t, u_t^{\text{b}}) \\
           {\cdot}\left[ R(x_t, u_t^{\text{b}})+ \lambda q(x_{t+1},u^{\text{b}}_t) \right], \quad \forall t \in \mathbb{N},
    \end{multline}
    for all possible realization $x_t \in \mathcal{X}$ and $u_t \in \mathcal{U}$, where $u^{\text{r}}= g^{\text{r}}(x_t)$ is the recommended action and $u^{\text{b}}= g^{\text{b}}(x_t)$ is the baseline action. 
   
\end{definition}
 
\begin{lemma}\label{lemma}
    The operator $\mathcal{J}$ is a contraction mapping. 
\end{lemma}
\begin{proof}
We prove the result using the definition of a contraction mapping, i.e., we prove that $\mathcal{J}$ satisfies
    \begin{equation}\label{lemma1_eq_1}
        \|\mathcal{J}q_1 - \mathcal{J}q_2\|_{\infty} \leq \lambda \|q_1 - q_2\|_{\infty},
    \end{equation}
    for any bounded $q: \mathcal{X}\times \mathcal{U} \to \mathbb{R}$ and for any possible realization $x_t \in \mathcal{X}$ and $u_t \in \mathcal{U}$.
   
Substituting \eqref{oper}, using the definition of the $\infty$-norm, in \eqref{lemma1_eq_1}, we obtain:
\begin{equation} \label{lemma1_eq_2}
\begin{aligned}
     &\|\mathcal{J}q_1 - \mathcal{J}q_2\|_{\infty} = 
     \max_{x,u^{\text{r}}}\Big\lvert \theta_{t+1} {\cdot} \lambda {\cdot} \hspace{-10pt} \sum_{x_{t+1} \in \mathcal{X}} \hspace{-10pt} P(x_{t+1}~|~x_t, u_t^{\text{r}}) \\
     & {\cdot} \left[ \max_{u^{\text{r}}}q_{1}(x_{t+1},u_t^{\text{r}}) \right. 
      - \left. \max_{u^{\text{r}}}q_{2}(x_{t+1},u_t^{\text{r}})\right]  +(1-\theta_{t+1}) {\cdot}\\
      &\lambda {\cdot} \sum_{x_{t+1} \in \mathcal{X}} P(x_{t+1}~|~x_t, u_t^{\text{b}}) {\cdot}
   \left[ q_1(x_{t+1},u^{\text{b}}) - q_2(x_{t+1},u^{\text{b}})\right]  \Big\lvert.
\end{aligned}
\end{equation}
From the triangle inequality, we obtain:
\begin{equation}\label{lemma1_eq_3}
    \begin{aligned}
          &\|\mathcal{J}q_1 - \mathcal{J}q_2\|_{\infty} \leq 
          \max_{x,u^{\text{r}}} \theta_{t+1} {\cdot} \lambda {\cdot} \sum_{x_{t+1} \in \mathcal{X}} P(x_{t+1}~|~x_t, u_t^{\text{r}})  \\
          & {\cdot}
          \Big\lvert\max_{s_{t+1},U^{\text{r}}_t}  q_{1}(s_{t+1},u_t^{\text{r}}) 
          -  \max_{s_{t+1},u^{\text{r}}_t}q_{2}(s_{t+1},u_t^{\text{r}})\Big\lvert +(1-\theta_{t+1}) {\cdot}\\
          &\lambda {\cdot} \sum_{X_{t+1} \in \mathcal{X}} P(x_{t+1}~|~x_t, u_t^{\text{b}}){\cdot}
        \Big\lvert q_1(x_{t+1},u^{\text{b}}) - q_2(x_{t+1},u^{\text{b}})\Big\lvert.
    \end{aligned}
\end{equation}
      Since $  \sum_{x_{t+1} \in \mathcal{X}} P(x_{t+1}~|~x_t, u_t^{\text{r}}) = 1$ and $\sum_{x_{t+1} \in \mathcal{X}} P(x_{t+1}~|~x_t, u_t^{\text{b}}) = 1$ in the RHS of \eqref{lemma1_eq_3},  $\|\mathcal{J}q_1 - \mathcal{J}q_2\|_{\infty} \leq \max_{x,u^{\text{r}}} \theta_{t+1} {\cdot} \lambda {\cdot} |q_{1}  -q_{2}| +(1-\theta_{t+1}) {\cdot} \lambda {\cdot} | q_1 -q_2 |$.
   
    We complete the proof using the definition for the $\infty$-norm to state that $\|\mathcal{J}q_1 - \mathcal{J}q_2\|_{\infty} \leq \theta_{t+1} {\cdot} \lambda {\cdot} \Big(\|q_{1} - q_{2}\|_{\infty}\Big)+(1-\theta_{t+1}) {\cdot} \lambda {\cdot} \Big(\| q_1 - q_2 \|_{\infty} \Big) = \lambda {\cdot} \| q_1 - q_2 \|_{\infty}$.
\end{proof}

Using the Banach fixed point theorem and the result of Lemma \ref{lemma}, the equation $Q = \mathcal{J}Q$ admits a unique solution $Q^{\infty} = \mathcal{J} Q^{\infty}$. Furthermore, starting at $Q^0(x_t,u_t) = 0$, the fixed point iteration around $\mathcal{J}$ generates a sequence of functions $Q^{k+1}(x_t,u_t) = \mathcal{J}Q^{k}(x_t,u_t) = \mathcal{J}^kQ^{0}(x_t,u_t)$ for all $k = 1,\dots, $ such that $lim_{k \to \infty} \mathcal{J}^kQ^{0} = Q^{\infty}$. 
Next, we prove that the operator is monotone and then use this to establish that the fixed point $Q^{\infty}$ is the optimal value.
\begin{lemma}\label{lemma2}
    Consider the maximum value over all control laws in Problem \ref{problem_1} given by
    \begin{gather} \label{Q_opt}
        Q^* = \max_{g^{\text{a}}} \mathbb{E}^{g^{\text{a}}}\Big[\sum_{t=1}^{\infty}\lambda^{t-1}R(X_t,U_t)~|~ X_1 = x\Big].
    \end{gather}
    Then, the fixed point solution $Q^{\infty} = \mathcal{J} Q^{\infty}$ also satisfies $Q^{\infty} = Q^*$.
\end{lemma}
\begin{proof}
    Suppose that $R(x_t,u_t)$ is bounded by $C\in\mathbb{R}$ for all possible states and actions,
    and consider a finite truncation as $Q^*_T = \max_{g^{\text{a}}} \mathbb{E}^{g^{\text{a}}}\Big[\sum_{t=1}^{T}\lambda^{t-1}R(X_t,U_t)|X_1 = x \Big]$. Case (1), by construction, the truncate $Q^*_T$ is sub-optimal w.r.t. $Q^*$, i.e., $Q^* \geq Q^*_T$. However, $Q^*_T$ satisfies $Q^*_T = \mathcal{J}^kQ^0$, and hence as $k \to \infty$ we can write $Q^* \geq Q^*_T = \mathcal{J}^kQ^0 = Q^{\infty}$. In Case (2), we use the assumption that the reward is bounded, and hence we can write $Q^* \leq Q^*_T + \sum_{t = T}^{\infty} \lambda^T C$ for all $T$. Taking the $\lim_{T \to \infty} $, we get $Q^* \leq Q^*_T = Q^{\infty} $. From Cases (1) and (2), we conclude that $Q^{\infty} = Q^*$.
\end{proof}

Next, we prove the convergence of Algorithm \ref{alg:cap} in the tabular setting with finite-valued random variables.
\begin{theorem}
     Given a finite MDP $(\mathcal{X},\mathcal{U},P,R,\lambda)$, the adherence-aware Q-learning algorithm, defined by the update rule \eqref{eq_adh_q}
      converges with probability 1 to the optimal $Q^*$ in \eqref{Q_opt} under the following conditions for all $(x,u) \in \mathcal{X} \times \mathcal{U}$:
        \begin{align}
            \text{1)} \quad &\sum_t \alpha_t(x,u) = \infty, \\
            \text{2)} \quad &\sum_t \alpha_t(x,u)^2 < \infty, \\
            \text{3)} \quad &R(x, u^\text{r}) \text{ and } R(x, u^{\text{b}}) \text{ are bounded.}
        \end{align}
        \label{theorem1}
\end{theorem}  
\begin{proof}
    For any $t \in \mathbb{N}$, we write the $t+1$-th iteration of the $Q$-update rule \eqref{eq_adh_q} as:    
    \begin{multline}
             Q_{t+1}(x,u) = (1-\alpha_t(x,u))Q_t(x,u) \\
             + \alpha_t(x,u) \Big\{ \theta_{t+1} 
             \left[\max_{u^{\text{r}}}(R(x,u^{\text{r}}) + \lambda {\cdot} Q_t(x',u^{\text{r}}))\right] \\
             + (1-\theta_{t+1}) \left[ R(x,u^{\text{b}}) + \lambda {\cdot} Q_t(x',u^{\text{b}}) \right] \Big\},
    \end{multline}
    where recall that $u^b = g^b(x)$.
    By subtracting $Q^*(x,u)$ from both sides, we get:
    \begin{multline}
            \Delta_{t+1}(x,u) = (1-\alpha_t(x_t,u_t)) {\cdot} \Delta_{t}(x,u)
            + \alpha_t(x,u) \\
            {\cdot} \Big\{ \theta_{t+1} {\cdot}
            \left[\max_{u^{\text{r}}}\big(R(x,u^{\text{r}}) + \lambda {\cdot}Q_t(x',u^{\text{r}})\big)\right] 
            + (1-\theta_{t+1}) \\
            {\cdot} \left[ R(x,u^{\text{b}}) + \lambda {\cdot} Q_t(x',u^{\text{b}}) \right] - Q^*(x,u)\Big\}, \label{eq:}
    \end{multline}
    where $\Delta_t(x,u) = Q_t(x,u) - Q^*(x,u)$.
    Let 
    \begin{align}
    F_t(x,u,x')  =  \theta_{t+1} {\cdot} \big[\max_{u^{\text{r}}}\big(R(x,u^{\text{r}}) + \lambda Q_t(x',u^{\text{r}})\big)\big]  \nonumber\\
    + (1-\theta_{t+1}) {\cdot}\big[ R(x,u^{\text{b}}) + \lambda {\cdot} Q_t(x',u^{\text{b}}) \big] - Q^*(x,u).\label{eq:Ftfunction}
    \end{align}
        
    Next, we prove that the expectation and the variance of $F_t(x,u,x')$ are bounded above.
    By taking the expectation in both sides in \eqref{eq:Ftfunction} given the history $\mathcal{F}_t = \{\Delta_t, \Delta_{t-1},\dots, F_{t-1}, \dots \}$, we obtain
    \begin{equation}
        \begin{aligned}
            &\mathbb{E}[ F_t(X,U,X')~\big|~\mathcal{F}_t,x,u] = 
            \mathbb{E}\Big\{\theta_{t+1} {\cdot} \Big[ \max_{u^{\text{r}}}\big(R(X,U^{\text{r}}) \\
            &+ \lambda {\cdot} Q_t(X',U^{\text{r}})\big)~\big|~\mathcal{F}_t,x,u \Big] \Big\}  
            +\mathbb{E}\Big\{(1-\theta_{t+1}) {\cdot} \Big[ R(X,U^{\text{b}}) \\
            &+ \lambda {\cdot} Q_t(X',U^{\text{b}})~\big|~\mathcal{F}_t,x,u \Big]\Big\} - Q^*(x,u).
        \end{aligned}
    \end{equation}
    From Lemma \ref{lemma}, the update rule for $\theta_{t+1}$ is independent of state and action so that we can write
    \begin{equation}
        \begin{aligned}
            &\mathbb{E}[ F_t(X,U,X')~\big|~\mathcal{F}_t,x,u] = 
            \mathbb{E}[\theta_{t+1}~\big|~\mathcal{F}_t] \\
            &{\cdot}\mathbb{E} \Big[\max_{u^{\text{r}}}\big(R(X,U^{\text{r}}) 
            + \lambda {\cdot}Q^n(X',U^{\text{r}})\big)~\big|~\mathcal{F}_t,x,u\Big]  \\
            &+\mathbb{E}[(1-\theta_{t+1})~|~\mathcal{F}_t] {\cdot} \mathbb{E}\Big[ R(X,U^{\text{b}}) \\ & +\lambda {\cdot}Q_t(X',U^{\text{b}})~\big|~\mathcal{F}_t,x,u \Big] - Q^*(x,u).
        \end{aligned}
    \end{equation}
   
    From \eqref{expectation}, the last equation becomes
    \begin{equation}
        \begin{aligned}
               &\mathbb{E}[ F_t(X,U,X')~\big|~\mathcal{F}_t,x,u] = 
                \theta{\cdot} \mathbb{E} \Big[\max_{u^{\text{r}}}\big(R(X,U^{\text{r}}) \\
                &+ \lambda {\cdot}Q_t(X',U^{\text{r}})\big)~\big|~\mathcal{F}_t,x,u\Big] 
                +(1-\theta){\cdot} \mathbb{E}\Big[ R(X,U^{\text{b}}) \\
                &+ \lambda{\cdot} Q_t(X',U^{\text{b}})~\big|~\mathcal{F}_t,x,u \Big] - Q^*(x,u).
        \end{aligned}
    \end{equation}
    
    By expanding the expectations we have
    \begin{equation}
        \begin{aligned}
             & \mathbb{E}[ F_t(X,U,X')~\big|~\mathcal{F}_t,x,u] = \theta{\cdot} \max_{u^{\text{r}}} \hspace{-5pt} \sum_{x' \in \mathcal{X}} \hspace{-5pt}P(x'~|~x, u^{\text{r}}){\cdot}
                \\
               &\Big[ \big(R(x,u^{\text{r}})+ \lambda Q_t(x',u^{\text{r}})\big)\Big] 
                +(1-\theta){\cdot} \hspace{-5pt} \sum_{x' \in \mathcal{X}}\hspace{-5pt} P(x'~|~x, u^{\text{b}}){\cdot} \hspace{-2pt}
                 \\
                & \Big[ R(x,u^{\text{b}}) + \lambda {\cdot}Q_t(x',u^{\text{b}}) \Big] - Q^*(x,u).
        \end{aligned}
    \end{equation}
    
    The first two terms in the RHS of the last equation form the contraction operator $\mathcal{J}$, hence from Lemma \ref{lemma2}, 
    \begin{equation} \label{exp}
        \mathbb{E}[ F_t(X,U,X')~\big|~\mathcal{F}_t,x,u] = (\mathcal{J}Q)(xu) - (\mathcal{J}Q^*)(x,u).
    \end{equation}
    To conclude the result, we use the $\infty$-norm on both sides of \eqref{exp} and the result of Lemma \ref{lemma}, and we obtain
 \begin{align}
\|\mathbb{E}[ F_t(X,U,X')~\big|~\mathcal{F}_t]\|_{\infty} &= \|(\mathcal{J}Q)(x,u) - (\mathcal{J}Q^*)(x,u)\|_{\infty} \nonumber\\
&\leq \lambda{\cdot}\|Q_t - Q^*\|_{\infty} = \lambda {\cdot} \|\Delta_t\|_{\infty}.
\end{align}
  
    Next, we show that the variance is also bounded above. By using the definition of the variance of $F_t(X,U,X')$ we write:    
    \begin{equation} \label{var_2}
        \begin{aligned}
            &var[F_t(X,U,X')~\big|~\mathcal{F}_t,x,u] = 
            \mathbb{E}\Big\{ \Big( \theta_{t+1}{\cdot} \Big(\max_{u^{\text{r}}}\big(R(X,U^{\text{r}})\\ &+\lambda {\cdot}Q_t(X',U^{\text{r}})\big)\Big) 
            +(1-\theta_{t+1}){\cdot} \Big( R(X,U^{\text{b}}) \\
            &+ \lambda{\cdot} Q_t(X,U^{\text{b}})\Big) - Q^*(X,U) \\
            &- \Big[\underbrace{(\mathcal{J}Q)(X,U) - (\mathcal{J}Q^*)(X,U)}_\text{$\mathbb{E}[ F_t(X,U,X')|\mathcal{F}_t,x,u]$} \Big]\Big)^2 ~\big|~\mathcal{F}_t,x,u\Big\}.
        \end{aligned}
    \end{equation}
    
From Lemma \ref{lemma2}, we can simplify the last term in the RHD of \eqref{var_2} as
\begin{equation}
    \begin{aligned}
        &var[F_t(X,U,X')|\mathcal{F}_t,x,u] = 
        \mathbb{E}\Big\{ \Big( \theta_{t+1} {\cdot}\Big(\max_{u^{\text{r}}}\big(R(X,U^{\text{r}})\\
        &+\lambda {\cdot}Q_t(X',U^{\text{r}})\big)\Big) 
        +(1-\theta_{t+1}){\cdot} \Big( R(X,U^{\text{b}}) \\
        &+ \lambda{\cdot} Q_t(X',U^{\text{b}})\Big)  -
        (\mathcal{J}Q)(X,U) \Big)^2 ~|~\mathcal{F}_t,x,u\Big\},
    \end{aligned}
\end{equation}
    which is the definition of the variance with variable $\theta_{t+1} {\cdot}\left(\max_{u^{\text{r}}}\big(R(X,U^{\text{r}}) + \lambda {\cdot} Q_t(X',U^{\text{r}})\big)\right) +(1-\theta_{t+1}){\cdot}( R(X,U) + \lambda {\cdot}Q_t(X',U^{\text{b}})$. 
    By using the Assumption (3), i.e., the reward. Hence its variance is bounded, i.e.,
    \begin{equation}
        var[F_t(X,U,X')~\big|~\mathcal{F}_t,x,u] \leq C{\cdot}(1+\|\Delta_t\|_{\infty})^2,
    \end{equation}
    and the proof is complete.
\end{proof}
\section{Numerical Examples} \label{sec:example}
\subsection{Inventory Example}
We consider an HDM in \textbf{inventory control} as a shop owner, where the shop provides multiple items $(K)$ for sale. The goal of the HDM is to maximize its cumulative revenue at each time step by deciding the number of items that need to be ordered based on the stochastic demand of the future. Let $x_k$ be the state of commodity $k$, where $x_k \in \{0,1,\dots, 100\}$ and $k = 1,2, \dots, K$. The action $u_k \in \{0,1,\dots, 100\}$ corresponds to the amount of the order, and $d_k \in \{0,1,\dots, 100\}$ corresponds to the amount of the demand that is stochastic and follows the uniform distribution with $B(0,100)$.
We assume that the baseline law of the HDM is according to $(s,S)$ method \cite{veinott1965computing}, where $s$ represents the threshold of which the HDM should order the amount of $S$. The state evolution for the inventory is according to $x_{k+1} = \max(0,x_k + u_k - D_k)$, and the reward is defined as the incomes of sales minus the holding cost $H$, and the ordering cost $C$. The holding cost is given by $H(x_k,u_k) = h {\cdot}\max(0, x_k - u_k)$, where $h>0$ while the ordering cost is given by $C(u_k) = c {\cdot}u_k$, where $c>0$. 
\subsection{Machine replacement}
We consider the \textbf{Machine replacemen}t MDP problem with ten states. The set of states is $\{1,2,3,4,5,6,7,8,S,L \}$ and the set of actions is $\{\text{repair}, \text{wait}\}$. Each state represents the condition of the machine, where in-state $8$ the machine is broken, and states $S, L$ represent the short and the long repair, respectively. To model this problem, we adopt the same rewards and transitions as in \cite{grand2022best}. In particular, we show the transition probabilities for both actions in Figure \ref{fig:transisiton}. Also, we define a reward of $18$ in state $S$, $20$ in state $L$, and $0$ in state $8$. All other states have a reward of $20$. We assume the baseline law is always $\text{wait}$ when the machine is not broken or in the long repair $L$. In any other case, the action is $\text{repair}$.
\begin{figure}[h!]
    \centering
    \includegraphics[width=\linewidth, keepaspectratio]{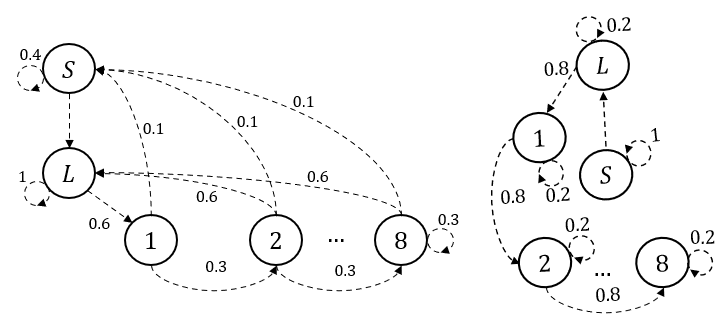}
    \caption{Transition probabilities for repair (left) and wait (right).}
    \label{fig:transisiton}
\end{figure}
    
\vspace{-3pt}
\subsection{Simulation results}
We considered the learning rate $\alpha$ and the discount factor $\lambda$ equal to $0.9$ for the simulation results. We also considered that the update rule of $\theta$ will converge to 0.7.
First, we show convergence for the two different demand distributions by tracking the initial state value for $10,000$ and $100$ time steps, respectively. In both numerical examples, the algorithm converges (Figs. \ref{fig:convergence_inv} and \ref{fig:convergence_repl}). For the inventory case, though, the algorithm needs more time to converge to optimal value due to the large state space.
Furthermore, in both scenarios, we compared our adherence-aware Q-learning algorithm with the regular Q-learning and the baseline law, as it was the only one the HDM would implement. Figures \ref{fig:rewards_inv} and \ref{fig:rewards_replacement} illustrate the result in which the average actual reward using the adherence-aware Q-learning algorithm is better than the other two. Finally, we investigated the algorithm's performance for different values of $\theta$. Figures  \ref{fig:thetas_repl} and \ref{fig:thetas_inv} illustrate the result and show that when $\theta \in [0,0.5]$ our approach is slightly better than the baseline law, while for $\theta \in [0.5,1]$ our approach outperforms the other two.
\begin{figure*}[h!]
    \centering
    \begin{subfigure}[t]{0.32\textwidth}
        \includegraphics[width=\textwidth]{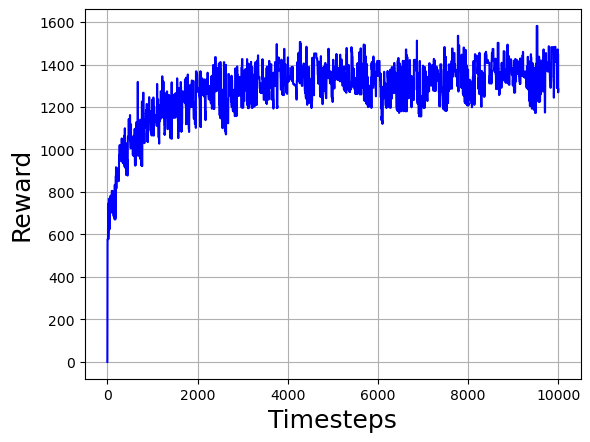}
        \caption{Convergence of the adherence-aware algorithm.}
        \label{fig:convergence_inv}
    \end{subfigure}
    \hfill
    \begin{subfigure}[t]{0.32\textwidth}
        \includegraphics[width=\textwidth]{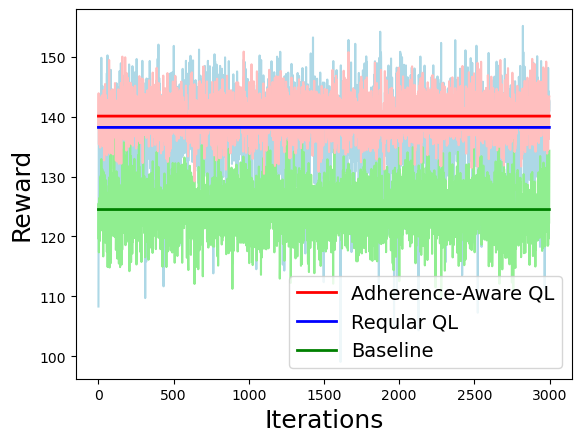}
        \caption{Actual reward for the three different approaches.}
        \label{fig:rewards_inv}
    \end{subfigure}
    \hfill
    \begin{subfigure}[t]{0.32\textwidth}
        \includegraphics[width=\textwidth]{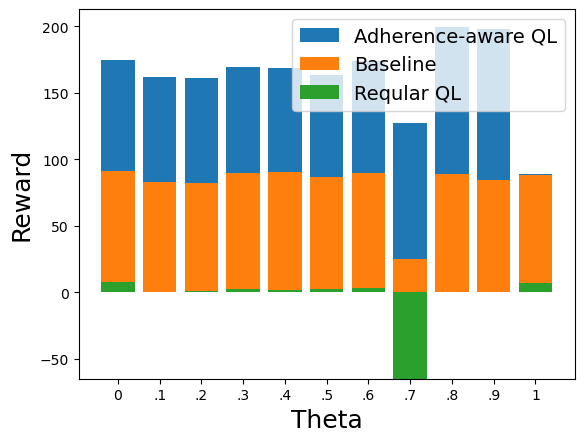}
        \caption{Influence of the parameter $\theta$.}
        \label{fig:thetas_repl}
    \end{subfigure}
    \caption{Inventory control results.}
    \label{fig:Potentials_parameters}
    \vspace{-5pt}
\end{figure*}

\begin{figure*}[h!]
    \centering
    \begin{subfigure}[t]{0.32\textwidth}
        \includegraphics[width=\textwidth]{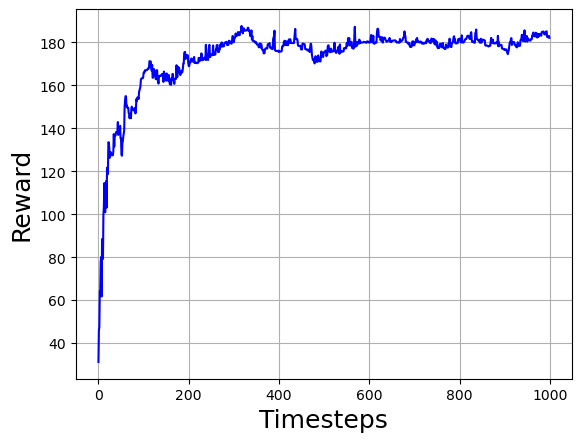}
        \caption{Convergence of the adherence-aware algorithm.}
        \label{fig:convergence_repl}
    \end{subfigure}
    \hfill
    \begin{subfigure}[t]{0.32\textwidth}
        \includegraphics[width=\textwidth]{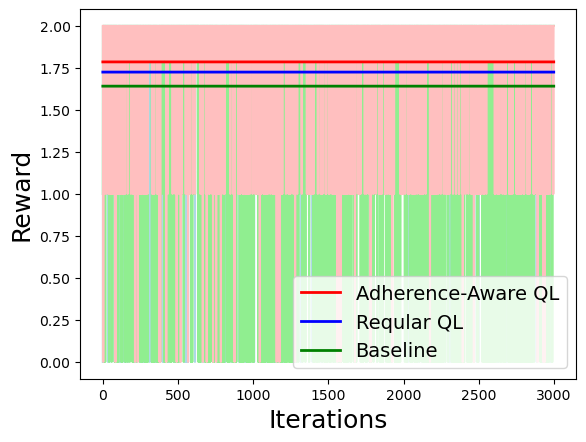}
        \caption{Actual reward for the three different approaches.}
        \label{fig:rewards_replacement}
    \end{subfigure}
    \hfill
    \begin{subfigure}[t]{0.32\textwidth}
        \includegraphics[width=\textwidth]{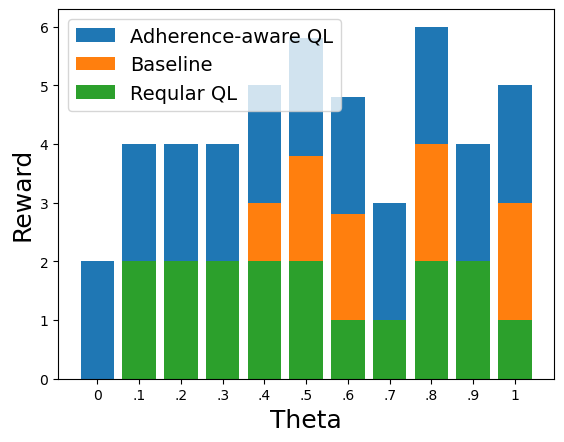}
        \caption{Influence of the parameter $\theta$.}
        \label{fig:thetas_inv}
    \end{subfigure}
    \caption{Machine replacement results.}
    \label{fig:Potentials_parameters}
    \vspace{-5pt}
\end{figure*}

\section{Concluding Remarks} \label{sec:conclusion}
In this letter, we proposed an ``adherence-aware Q-learning" designed to derive optimal recommendation actions for HDMs. Our approach considers the complexity of the problem, where both the dynamics of the environment and the level of adherence to recommendations remain unknown.
The structure of our algorithm is based on a combination of the HDM's baseline law and an update rule for estimating the adherence level to the recommendations. We proved the convergence of the adherence-aware Q-learning algorithm to the optimal value function, and we applied this algorithm to two numerical examples, illustrating its ability to converge to the optimal value and outperform alternative methods in various system scenarios.
Future work should consider situations where the baseline law is unknown and can be learned over time or the HDM has partial observability of the system state. Extending the results in a team of human-driven vehicles with a decentralized information structure \cite{Malikopoulos2021} should be also a potential direction of future research.

\bibliographystyle{IEEEtran}
\bibliography{Faros, IDS}

\end{document}